\documentclass{article}

\usepackage{microtype}
\usepackage{graphicx}
\usepackage{caption}

\usepackage{subcaption}
\usepackage{booktabs} 

\usepackage{hyperref}

\usepackage[accepted]{icml2024}

\usepackage{amsmath}
\usepackage{amssymb}
\usepackage{mathtools}
\usepackage{amsthm}
\usepackage{bm}
\usepackage[capitalize,noabbrev]{cleveref}

\theoremstyle{plain}
\newtheorem{theorem}{Theorem}[section]
\newtheorem{proposition}[theorem]{Proposition}

\theoremstyle{definition}
\newtheorem{definition}[theorem]{Definition}

\theoremstyle{remark}

\Crefname{equation}{}{}

\usepackage[textsize=tiny]{todonotes}

\icmltitlerunning{Efficient Pareto Manifold Learning with Low-Rank Structure}

\begin{document}

\twocolumn[
\icmltitle{Efficient Pareto Manifold Learning with Low-Rank Structure}


\icmlsetsymbol{equal}{*}

\begin{icmlauthorlist}
\icmlauthor{Weiyu Chen}{hkust}
\icmlauthor{James T. Kwok}{hkust}
\end{icmlauthorlist}

\icmlaffiliation{hkust}{Department of Computer Science and Engineering, The Hong Kong University of Science and Technology}

\icmlcorrespondingauthor{Weiyu Chen}{wchenbx@cse.ust.hk}

\icmlkeywords{multi-objective optimization, multi-task learning}

\vskip 0.3in
]



\printAffiliationsAndNotice{}  

\newcommand{\E}{\mathbb{E}}
\newcommand{\N}{\mathbb{N}}
\newcommand{\R}{\mathbb{R}}
\newcommand{\V}{\mathbb{V}}

\newcommand{\X}{\mathcal{X}}
\newcommand{\B}{\mathcal{B}}

\newcommand{\bone}{\boldsymbol{1}}

\newcommand{\ba}{\boldsymbol{a}}
\newcommand{\bb}{\boldsymbol{b}}
\newcommand{\bd}{\boldsymbol{d}}
\newcommand{\bg}{\boldsymbol{g}}
\newcommand{\bh}{\boldsymbol{h}}
\newcommand{\bp}{\boldsymbol{p}}
\newcommand{\bq}{\boldsymbol{q}}
\newcommand{\br}{\boldsymbol{r}}
\newcommand{\bt}{\boldsymbol{t}}
\newcommand{\bx}{\boldsymbol{x}}
\newcommand{\by}{\boldsymbol{y}}
\newcommand{\bu}{\boldsymbol{u}}
\newcommand{\bv}{\boldsymbol{v}}
\newcommand{\bw}{\boldsymbol{w}}
\newcommand{\bz}{\boldsymbol{z}}

\newcommand{\bZ}{\boldsymbol{Z}}
\newcommand{\bD}{\boldsymbol{D}}
\newcommand{\bF}{\boldsymbol{F}}
\newcommand{\bG}{\boldsymbol{G}}
\newcommand{\bA}{\boldsymbol{A}}
\newcommand{\bB}{\boldsymbol{B}}
\newcommand{\bC}{\boldsymbol{C}}
\newcommand{\bM}{\boldsymbol{M}}
\newcommand{\bN}{\boldsymbol{N}}
\newcommand{\bR}{\boldsymbol{R}}
\newcommand{\bS}{\boldsymbol{S}}
\newcommand{\bI}{\boldsymbol{I}}
\newcommand{\bW}{\boldsymbol{W}}
\newcommand{\bU}{\boldsymbol{U}}

\newcommand{\balp}{\boldsymbol{\alpha}}
\newcommand{\bxi}{\boldsymbol{\xi}}
\newcommand{\bnu}{\boldsymbol{\nu}}
\newcommand{\blam}{\boldsymbol{\lambda}}
\newcommand{\bpi}{\boldsymbol{\pi}}
\newcommand{\bthe}{\boldsymbol{\theta}}
\newcommand{\bepsilon}{\boldsymbol{\epsilon}}

\newcommand{\tcb}{\textcolor{blue}}
\newcommand{\norm}[1]{\left\|#1\right\|}
\newcommand{\inner}[2]{\left\langle#1, #2\right\rangle}

\begin{abstract}
Multi-task learning, which optimizes performance across multiple tasks, is inherently a multi-objective optimization problem. Various algorithms are developed to provide discrete trade-off solutions on the Pareto front. Recently, continuous Pareto front approximations using a linear combination of base networks have emerged as a compelling strategy. However, it suffers from scalability issues when the number of tasks is large. To address this issue, we propose a novel approach that integrates a main network with several low-rank matrices to efficiently learn the Pareto manifold. It significantly reduces the number of parameters and facilitates the extraction of shared features. We also introduce orthogonal regularization to further bolster performance. Extensive experimental results demonstrate that the proposed approach outperforms state-of-the-art baselines, especially on datasets with a large number of tasks.
\end{abstract}

\section{Introduction}
Multi-task learning (MTL) \cite{mtl, mtl2} endeavors to efficiently and effectively learn multiple tasks, leveraging shared information to enhance overall performance. 
As tasks may have disparate scales and potentially conflicting objectives,
the challenge of balancing the tasks is a critical aspect of MTL.

Building on the seminal work that conceptualizes MTL as a multi-objective optimization (MOO) problem \cite{mtlmo}, a suite of algorithms has been developed to obtain trade-off solutions along the Pareto front (PF). Examples include EPO \cite{EPO}, PMTL \cite{PMTL}, CAGrad \cite{CAGrad}, and Nash-MTL \cite{NashMTL}. 
However, 
these methods
are limited to producing a finite set of discrete trade-off solutions, lacking the flexibility to adapt to varying user preferences.

Continuous approximation of the PF aims to provide an arbitrary number of solutions on the PF based on the user-provided preference vector. Pareto hypernetwork-based methods \cite{PHN, PHNHVI} utilize a hypernetwork \cite{hypernetworks} to learn the parameter space of a base network. This hypernetwork accepts preference vector as input and outputs the corresponding parameter of base network. Despite its efficacy, the size of the hypernetwork significantly exceeds that of the base network, thus restricting its use to small base networks. COSMOS \cite{COSMOS} attempts to circumvent this limitation by incorporating preference vectors directly into the base network as an extra input. However, it falls short of achieving satisfactory results due to the limited parameter space.

Pareto Manifold Learning (PaMaL) \cite{PAMAL}
proposes to find a collection of base networks whose linear combination in the parameter space based on the preference vectors yields the corresponding Pareto-optimal solution. This approach exhibits promising performance. However, 
as we need to maintain a separate base network for each task,
it incurs a corresponding surge in the number of parameters
when we have a large number of tasks.
Moreover, base networks cannot learn from each other
during training. This lack of sharing poses challenges when the ensemble comprises a large number of base networks.

To surmount these challenges, we introduce a novel approach that employs a main network and multiple low-rank matrices, which can significantly reduce the number of parameters when the number of tasks is large. Moreover, this design facilitates the extraction of shared features through the main network while simultaneously capturing task-specific differences via the low-rank matrices. We further integrate orthogonal regularization to promote disentanglement of these matrices. Empirical investigations demonstrate that the proposed algorithm outperforms PaMaL and other baselines, particularly when
the number of tasks
is large.

\textbf{Notations.} We denote the set $\{1,\ldots, m\}$ as $[m]$, and the $m$-dimensional simplex $ \{\balp \;|\; \sum_{i=1}^{m}\alpha_i = 1, \alpha_i \geq 0\}$ as $\Delta^m$. The norm $\|\cdot\|$ refers to the Euclidean norm when applied to vectors, and to the Frobenius norm when applied to matrices.

\section{Background}
\subsection{Multi-Objective Optimization}
A multi-objective optimization (MOO) \cite{moo} problem can be formulated as
        \begin{align} \label{eq:moo}
        \min_{\bthe \in \R^{d \times k}} ~~ \bF(\bthe) & = [f_1(\bthe), \ldots, f_m(\bthe)]^\top,
        \end{align}
where $m \geq 2$ is the number of objectives, and $d \times k$ is the shape of the parameter matrix. 
In this paper,
we represent the parameters in matrix form for convenience when discussing low-rank approximations.
\begin{definition}[Pareto Dominance and Pareto-Optimal \cite{moo}] 
A solution $\bthe_1$ is {\em dominated} by another solution $\bthe_2$ if and only if  $f_i(\bthe_1) \geq f_i(\bthe_2)$ for $i \in [m]$, and $\exists i \in [m], f_i(\bthe_1) > f_i(\bthe_2)$. A solution $\bthe^*$ is {\em Pareto-optimal} if and only if it is not dominated by any other $\bthe'$.
\end{definition}
A {\em Pareto set} (PS) is the set of all Pareto-optimal solutions. The {\em Pareto front} (PF) refers to the functional values associated with the solutions in the PS.

\subsection{Multi-Task Learning}
Multi-task learning (MTL) \cite{mtl, mtl2} refers to learning multiple tasks simultaneously.
In this paper, we focus on the optimization part, adopting the widely utilized shared-bottom architecture \cite{mtl} as the base network. More specifically, for a $m$-task network, we have a shared-bottom $\bthe^{sh}$ and $m$ task-specific heads $\bthe^{t_1}, \ldots, \bthe^{t_m}$. Throughout this paper, our primary focus is on the shared-bottom $\bthe^{sh}$. Hence, any mention of $\bthe$ hereafter refers specifically to $\bthe^{sh}$.

MTL can be viewed as a MOO problem \cite{mtlmo}. To deal with the possible conflicts of different tasks on the shared bottom, various algorithms have been proposed, such as MGDA \cite{mtlmo, MGDA}, PCGrad \cite{pcgrad}, IMTL \cite{imtl}, Graddrop \cite{graddrop}, RLW \cite{rlw}, CAGrad \cite{CAGrad}, Nash-MTL \cite{NashMTL}, and Auto-$\lambda$ \cite{autol}. These approaches, however, predominantly focus on obtaining a single solution rather than capturing the whole PF. 

\textbf{Discrete Approximation of Pareto Front.} \citet{PMTL} propose to generate a set of solutions to approximate the Pareto front. Subsequent research have introduced innovative strategies to enhance the discrete approximation. Notable among these are EPO \cite{EPO}, MOO-SVGD \cite{MOOSVGD}, GMOOAR \cite{gmooar} and PNG \cite{PNG}. These methods are inherently limited to producing a fixed set of discrete points on the PF. \citet{paretoexploration}
propose an algorithm to expand these discrete solutions in their vicinity. However, it can only produce PF segments around each discrete solution instead of the whole continuous PF.

\textbf{Continuous Approximation of Pareto Front.}
Continuous approximation of PF enables the derivation of solutions tailored to user preferences, offering an arbitrary resolution of the trade-off solutions.
\citet{PHN} introduces the concept of a Pareto hypernetwork (PHN), which learns a hypernetwork \cite{hypernetworks} that takes a preference vector as input and generates the corresponding Pareto-optimal network parameters. PHN-HVI \cite{PHNHVI} further 
develops the method using hypervolume maximization \cite{HIGA}. However, scalability of the approach becomes a concern when applying the hypernetwork to large networks, as the size of the hypernetwork is usually much larger than that of the base network itself. 

Conditioned one-shot multi-objective search (COSMOS) \cite{COSMOS} incorporates the preference vector as an additional input, enabling the generation of outputs with different preferences while minimizing the introduction of extra parameters. \citet{yoto} propose the use of FiLM layers \cite{film} to incorporate preferences through channel-wise multiplication and addition operations on the feature map. However, these methods are sometimes constrained by the relatively limited parameter space, which can lead to suboptimal performance.

Pareto manifold learning (PaMaL) \cite{PAMAL} 
aims to jointly learn multiple base networks $\bthe_1, \ldots, \bthe_m$, such that for any given preference vector $\balp \in \Delta^m$, the corresponding weighted combination of base networks leads to the Pareto-optimal model:
\begin{align}
  \bthe(\balp) = \sum_{i=1}^m \alpha_i \bthe_i.
\label{eq:pamal}
\end{align}
PaMaL demonstrates commendable performance. However, it only focuses on two or three tasks. When dealing with a large number of tasks, the number of parameters increases significantly as each task $i$ requires its own network $\bthe_i$. Moreover, the base networks cannot benefit from each other during training, which can potentially impair performance, especially when scaling to a larger number of base networks.

Learning a continuous PF is also explored in Bayesian optimization \cite{psl} and evolutionary optimization \cite{evopsl} for engineering problems with number of parameters up to 7. They fail to converge in deep learning scenarios due to the significantly larger parameter spaces.

\subsection{Low-Rank Adaptation}
Low-rank adaptation (LoRA) has emerged as a popular approach in the field of finetuning large pre-trained models \cite{LoRA}. It only finetunes a low-rank matrix instead of the entire model, resulting in improved efficiency and the prevention of overfitting. Some subsequent studies \cite{TA, tiesmerge, Tangent} explore the merging of LoRA modules trained on different tasks. These works focus on direct addition or subtraction without incorporating any training. 
\citet{lora_acl} proposes fine-tuning a pretrained model for multiple tasks using a single LoRA module (with some task-specific parameters). \citet{lora_otrh} proposes fine-tuning a pretrained model by adding LoRA modules one by one, where the newly-added LoRA module is orthogonal to the previous ones. 

The main difference between these methods and the proposed method is that they learn a single LoRA each time and output a single fine-tuned network, while we train multiple low-rank matrices simultaneously and output a subspace of networks. To mitigate dominance of unitary transforms in weight update as observed in LoRA,
\citet{multilora} propose to split the LoRA module to multiple LoRA modules.
However, none of the existing research considers joint training of an arbitrary convex combination of low-rank matrices to learn a model subspace.

\section{Proposed Method} 
In this section, we introduce a novel approach that involves learning a main module and $m$ low-rank matrices for each layer (\cref{sec:lrs}).
Orthogonal regularization 
is further introduced
in \cref{sec:or}. Finally, the whole algorithm is presented in \cref{sec:alg}.

\subsection{Low-Rank Structure}
\label{sec:lrs}
We start by examining the similarities of the base networks
obtained by PaMaL \cite{PAMAL} on {\it MultiMNIST} \cite{multimnist}, using the same experimental setting as PaMaL. \cref{fig:similarities} shows the layer-wise cosine similarities 
between two base networks' parameters throughout training. As can be seen,
the similarities are almost zero
at the beginning of training.
As training progresses, a marked increase in similarity is observed. Furthermore, similarities are higher at the lower-level layers.
This motivates us to consider reducing the redundancy of base networks.

Consider the $m$ base networks in \cref{eq:pamal}. Denote the parameters in the $l$th layer of base network $\bthe_i$ by $\bthe_i^l$. In the $l$th layer, \cref{eq:pamal} can be written as:
\begin{align}
  \bthe(\balp)^l = \sum_{i=1}^m \alpha_i \bthe_i^l.
\label{eq:pamal_l}
\end{align}
Let $\bthe_0^l = \frac{1}{m}\sum_{i=1}^m \bthe_i^l$. We can rewrite \cref{eq:pamal_l} as:
\begin{align}
  \bthe(\balp)^l = \sum_{i=1}^{m} \alpha_i \left(\bthe_0^l + \frac{1}{m}\sum_{j=1}^{m}(\bthe_i^l - \bthe_j^l) \right).
  \label{eq:sum_minus}
\end{align}
Given the similarities observed from \cref{fig:similarities}, we approximate the difference between any two modules $\bthe_i^l - \bthe_j^l$ by a low-rank matrix.
Then, we use 
the low-rank matrix 
$\bB_i^l \bA_i^l$, to replace $\frac{1}{m}\sum_{j=1}^{m}(\bthe_i^l - \bthe_j^l)$ in \cref{eq:sum_minus},
leading to:
\begin{align}
  \bthe(\balp)^l = \sum_{i=1}^{m} \alpha_i (\bthe_0^l + \bB_i^l \bA_i^l).
  \label{eq:subs_lora}
\end{align}
We further add a scaling factor $s$ to regulate the significance of the low-rank component, as:
\begin{align}
  \bthe(\balp)^l = \bthe_0^l + s \sum_{i=1}^m \alpha_i \bB_i^l\bA_i^l.
  \label{eq:lora_pamal}
\end{align}

Through the above transformation, instead of learning $m$ base module in \cref{eq:pamal_l}, we learn a main module $\bthe_0^l$ and $m$ low-rank matrices $\bB_1^l\bA_1^l, \ldots, \bB_m^l\bA_m^l$. 
The main modules are expected to capture the common features across multiple tasks, thereby providing a shared foundation that each low-rank adaptation can leverage. 
An illustrative example of the proposed method is shown in Figure \ref{fig:illus}.  

\begin{figure*}
\begin{minipage}{.36\textwidth}
  \centering
    \includegraphics[width=\textwidth,trim={10, 8, 10, 20}, clip]{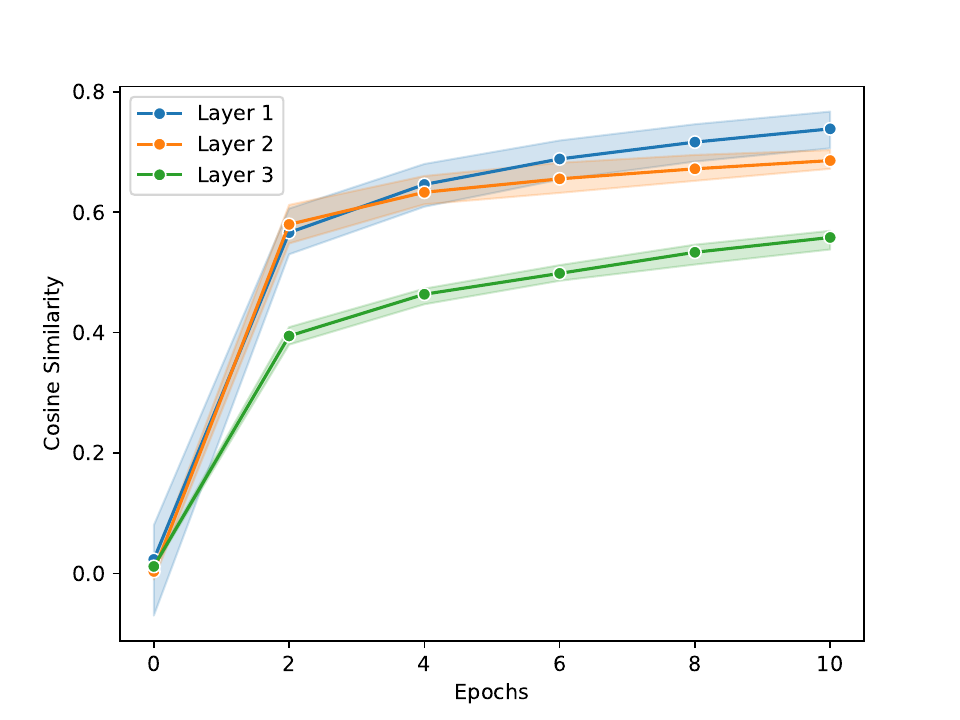}
      \vspace{-.1in}
      \caption{Layer-wise similarities between base networks obtained by PaMaL on {\it MultiMNIST} over three random seeds. Shaded areas represent the 95\% confidence interval.}
      \label{fig:similarities}
\end{minipage}   
\hfill
     \begin{minipage}{.6\textwidth}
     \centering
      \includegraphics[width=\textwidth]{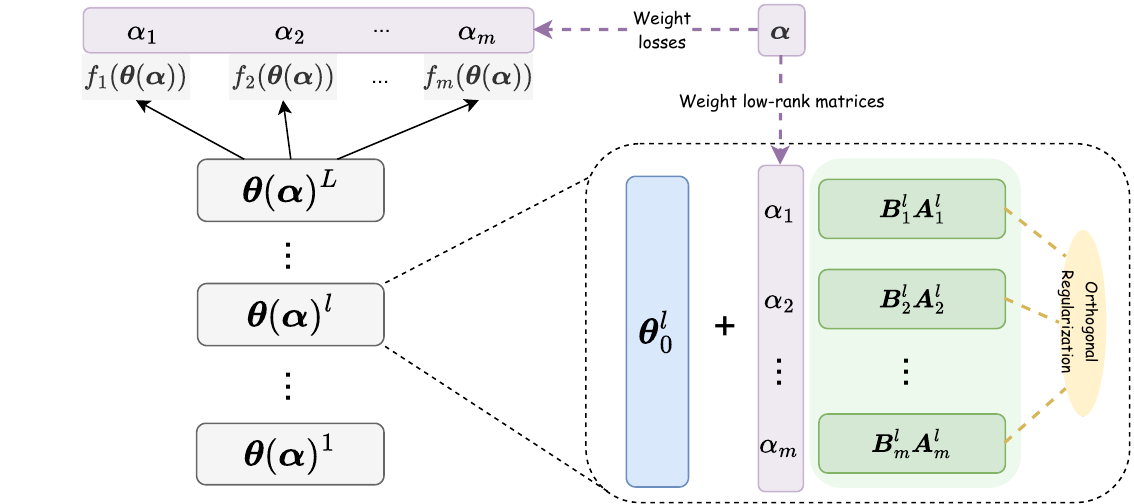}
      \caption{Illustration of the proposed LORPMAN on a $L$-layer base network with $m$ tasks. For each layer, we aim to learn $m$ low-rank matrices which are orthogonal to each other.}
      \label{fig:illus}
      \end{minipage}
\end{figure*}
\subsubsection{Computational Efficiency}
In the $l$th layer, 
let $\bB_i^l \in \R^{d^l \times r^l}$ and $\bA_i^l \in \R^{r^l \times k^l}$.
The proposed approach reduces the number of parameters to only $d^l k^l + m (d^l r^l + r^l k^l)$. Notably, the rank $r^l$ is usually set to be much smaller compared to $k^l$ and $d^l$. 
As a result, the proposed method achieves higher parameter efficiency compared to PaMaL, which has $m d^l k^l$ parameters. This difference becomes more significant
when dealing with larger networks and a larger number of tasks.

\subsubsection{Approximation Power}
Similar to PaMaL \cite{PAMAL},
we show 
in this section
the approximation power of the proposed structure.
While PaMaL  only considers two tasks,
we consider the more general case with $m$ tasks. Denote the optimal mapping from network input $\bx \in X \subset \R^u$ and preference vector $\balp \in \Delta^m$ to the corresponding point on the PF as $t(\bx, \balp): X \times \Delta^m \rightarrow \R^m$.
We have the following Theorem.
\begin{theorem}
\label{theo:approx}
Assume that $X \times \Delta^m$ is compact and $t(\bx, \balp)$ is continuous. For any $\epsilon > 0$, there exists a ReLU MLP $h$ with main network $\bthe_0$ and $m$ low-rank matrices $\bB_1\bA_1, \ldots, \bB_m\bA_m$, such that $\forall \bx \in X, \forall \balp \in \Delta^m$,
\begin{align*}
  \norm{t(\bx, \balp)-h\left(\bx; \bthe_0 + \sum_{i=1}^{m}\alpha_i \bB_i\bA_i\right)}\leq \epsilon.
\end{align*}
\end{theorem}
The proof is in \cref{sec:proof}. This Theorem shows that given any preference vector $\balp$, the proposed main network with low-rank matrices can approximately output the corresponding point on the PF within any given error margin $\epsilon$.

\subsection{Orthogonal Regularization}
\label{sec:or}

Orthogonal regularization, which encourages rows or columns of weight matrices to be approximately orthogonal, has been employed in various scenarios to regulate the parameters in a neural network \cite{othreg1, orthogonality, othreg2}. 
Here, we propose to extend it for orthogonality among the low-rank matrices.

First, we flatten each low-rank matrix into a 1-dimensional vector and normalize:
$  \bw^l_i = \frac{\texttt{flatten}(\bB_i^l\bA_i^l)} { \norm{\texttt{flatten}(\bB_i^l\bA_i^l)}}$.
Next, we concatenate them to construct a $(d^lk^l) \times m$ matrix $\bW^l = \texttt{concatenate}(\bw^l_1, \ldots, \bw^l_m)$. Our objective is to encourage orthogonality among the matrix columns. This orthogonality loss can be computed as:
$  \mathcal{R}_o^l = \norm{(\bW^l)^\top\bW^l - \bI}_2^2$,
where $\bI$ is the $m \times m$ identity matrix.
Finally, we compute the loss for the entire network as:
\begin{align}
  \mathcal{R}_o = \frac{1}{L}\sum_{l=1}^L \mathcal{R}_o^l,
\label{eq:orth_loss}
\end{align}
where $L$ is the number of layers.

The objective of the orthogonal regularization is to reduce redundancy and conflicts between different low-rank matrices. By reducing redundancy, the model encourages learning common features in the main network and task-specific features in the low-rank matrices, which can lead to more efficient parameter usage and better performance. 

Note that the computational complexity of the loss calculation is $O(m^2)$, which becomes expensive for large $m$. To address this, we use stochastic approximation when $m$ exceeds 3. Specifically, in each iteration, a subset of tasks $\mathcal{T}$, with cardinality 3, is randomly selected from the set of all $m$ tasks. We then construct $\hat{\bW}^l$ by concatenating $\{\bw^l_i\}_{i \in \mathcal{T}}$. Then, $\mathcal{R}_o^l$ is estimated as:
\begin{align}
  \mathcal{R}_o^l = \norm{(\hat{\bW}^l)^\top\hat{\bW}^l - \bI}_2^2,
\label{eq:orth_loss_l_est}
\end{align}
where $\bI$ is the $3\times3$ identity matrix.

This stochastic approximation ensures the complexity is independent of $m$. Empirical evaluations in \cref{sec:large} 
demonstrate that this still maintains satisfactory performance.

\subsection{Optimization}
\label{sec:alg}
To learn the Pareto manifold, we minimize the expectation of loss given an $\balp$ over the Dirichlet distribution $Dir(\bp)$. The optimization objective (without regularization) can be written as:
\begin{equation}
\label{eq:obj}
        \min_{\boldsymbol{\theta}_0, \mathbf{B}_1\mathbf{A}_1, \ldots, \mathbf{B}_m\mathbf{A}_m} \mathbb{E}_{\boldsymbol{\alpha}}
        \mathbb{E}_{\boldsymbol{\xi}} \left[\sum_{i=1}^{m}\alpha_i f_i(\boldsymbol{\theta}(\boldsymbol{\alpha});\boldsymbol{\xi}) \right],
\end{equation}
where $\bxi$ is a mini-batch of multi-task training data $\{(x_i^1, x_i^2, \ldots, x_i^m, y_i^1, y_i^2, \ldots, y_i^m)\}_{i=1}^q$ and $q$ is the batch size.
For any given $\balp$, we aim to minimize $\mathbb{E}_{\boldsymbol{\xi}} [\sum_{i=1}^{m}\alpha_i f_i(\boldsymbol{\theta}(\boldsymbol{\alpha});\boldsymbol{\xi})]$, which is the linear scalarization of the $m$ objective functions. 
Note that the solution obtained by linear scalarization is Pareto-optimal, which can be formally stated as follows:
\begin{proposition}[\cite{convex}]
\label{prop:po}
Given any $\balp \in \{\balp ~|~ \alpha_i > 0\}$, the optimal solution of the scalarized problem $\min_{\bthe} \sum_{i=1}^m \alpha_i f_i(\bthe)$ is a Pareto-optimal solution of the original multi-objective optimization problem \cref{eq:moo}.   
\end{proposition}
For the reader's convenience, the proof
is reproduced in \cref{sec:proof_po}. Thus, when \cref{eq:obj} is minimized, the obtained $\boldsymbol{\theta}(\boldsymbol{\alpha})$'s are Pareto-optimal solutions. Here, we only consider linear scalarization. More other scalarization methods \cite{tch, mtlmo, NashMTL}
can be considered.

The proposed algorithm,
which is called \underline{LO}w-\underline{R}ank \underline{P}areto \underline{MAN}ifold Learning (LORPMAN), is shown in  \cref{alg:pml}.
The training process is divided into two phases: First, we adapt both the main model $\boldsymbol{\theta}_0$ and low-rank matrices. After a certain number of epochs (which is referred to as \textit{freeze epoch}), we fix
the main model and only adapt the low-rank matrices. This encourages the low-rank matrices to learn task-specific representations instead of always relying on the main model to improve performance.

In each iteration, we sample $b$ preference vectors 
$\{\balp^1,\dots,
\balp^b\}
$
from the Dirichlet distribution. For each 
$\balp^i$, 
we compute the corresponding network parameters according to \cref{eq:lora_pamal}. Then, 
following \cite{PAMAL},
we calculate the multi-forward regularization loss $\mathcal{R}_p$ which penalizes incorrect solution ordering on the PF (details in \cref{sec:mf_reg}). Subsequently, we incorporate the orthogonal loss in \cref{eq:orth_loss} to encourage orthogonality among the different low-rank matrices, as discussed in the previous section.

\begin{algorithm}[tb]
  \caption{LORPMAN.}
  \label{alg:pml}
\begin{algorithmic}
  \STATE {\bfseries Input:} learnable main model parameters $\bthe_0$, learnable matrix $\{\bA_i, \bB_i\}_{i=1}^m$, distribution parameters $\bp$, window size $b$, batch size $q$, regularization coefficients $\lambda_p, \lambda_o$, scaling parameter $s$.
  \WHILE{not converged}
  \STATE sample a minibatch of multi-task training data $\boldsymbol{\xi}=\{(x_i^1, x_i^2, \ldots, x_i^m, y_i^1, y_i^2, \ldots, y_i^m)\}_{i=1}^q$;
  \STATE independently sample $\balp^1, \ldots, \balp^b$ from $\text{Dir}(\bp)$;
  \STATE compute corresponding model parameter for each $\balp^j$\\ 
  $\bthe(\balp^j)^l = \bthe_0^l + s \sum_{i=1}^m \alpha_i^j \bB_i^l\bA_i^l$;
  \STATE compute multi-forward regularization loss $\mathcal{R}_p$;
  \STATE compute orthogonal loss $\mathcal{R}_o$ in Eq. (\ref{eq:orth_loss});
  \STATE compute loss $\mathcal{L} = \sum_{j=1}^{b}\sum_{i=1}^{m}\alpha^j_i f_i(\bthe(\balp^j);\bxi) + \lambda_p \mathcal{R}_p + \lambda_o \mathcal{R}_o$;
  
  \IF {current epoch $<$ freeze epoch}
  \STATE take a gradient descent step on $\bthe_0$ and $\{\bA_i, \bB_i\}_{i=1}^m$;
  \ELSE
  \STATE take a gradient descent step on $\{\bA_i, \bB_i\}_{i=1}^m$;
  \ENDIF
  
  \ENDWHILE
\end{algorithmic}
\end{algorithm}

\section{Experiments}
In this section, we first demonstrate the training process of the proposed algorithm on a toy problem (\cref{sec:illex}). Then, we perform experiments on datasets with two or three tasks (\cref{sec:two_three}). Next, we scale up the number of tasks to up to 40 (\cref{sec:large}). We then compare the proposed method with algorithms that generate 
discrete PFs  (\cref{sec:discrete_baseline})
and algorithms that can generate a 
single solution 
(\cref{sec:single_baseline}).
Finally, ablation studies are presented in \cref{sec:ablation}. 

We adopt the Hypervolume (HV)
\cite{hv1, hv2} as performance indicator, which is widely used and can evaluate both the convergence and diversity of the obtained PF. Details of the HV indicator are in \cref{sec:hv}. To evaluate the performance of the obtained PFs, following \cite{PAMAL}, we sample 11 solutions on the obtained PF in the two-objective cases, 66 solutions in the three-objective cases, and 100 solutions when there are more than three objectives. We tune the hyperparameters according to the HV value on the validation datasets.

We compare the proposed method with state-of-the-art 
continuous PF approximation algorithms, namely PHN \cite{PHN}, PHN-HVI
\cite{PHNHVI}, COSMOS \cite{COSMOS}, and PaMaL \cite{PAMAL}. Since PHN and PHN-HVI require training a hypernetwork with significantly more parameters than the base network, they 
can only be run on \textit{MultiMNIST} and \textit{Census}.

\subsection{Illustrative Example}
\label{sec:illex}
To demonstrate the training process of the proposed algorithm, we employ a widely used two-objective toy problem \cite{pcgrad, CAGrad, NashMTL}, with parameter $\bthe = [\theta^1, \theta^2] \in \R^2$. The detailed definition of the toy problem is in \cref{sec:toy}.

We initialize $\bthe_0$ to $[4.5, 4.5]$. Since the parameters of the toy problem do not have a matrix structure, instead of using $\bB_i\bA_i$s for low-rank approximation, we use $\Delta\bthe_i \in \R^2$, where 
the second component is fixed to mimic the low-rank approximation constraint.  
We initialize $\Delta\bthe_1 = [-4.5, 0.0]$ and $\Delta\bthe_2=[4.5, 0.0]$. Given an $\balp$, the corresponding Pareto optimal parameter is 
$\bthe(\balp) = \bthe_0 + \alpha_1\Delta\bthe_1 + \alpha_2 \Delta\bthe_2$.

\cref{fig:toy} shows the trajectories of $\bthe$
obtained by LORPMAN
with $\balp = [0.5, 0.5]$, $[1, 0]$, and $[0,1]$. 
We can see that $\bthe$
reaches the middle of the PF
for $\balp = [0.5, 0.5]$, and 
to the two extreme points of the PF when
$\balp = [1, 0]$ and $[0,1]$. From \cref{fig:toy_ps}, we can see that these three solutions are on the same line in parameter space and we can obtain the whole Pareto set by varying $\balp$.

\begin{figure}[tbp]
  \centering
  \begin{subfigure}{0.45\columnwidth}
      \centering
      \includegraphics[width=\textwidth, trim={8, 8, 8, 8}, clip]{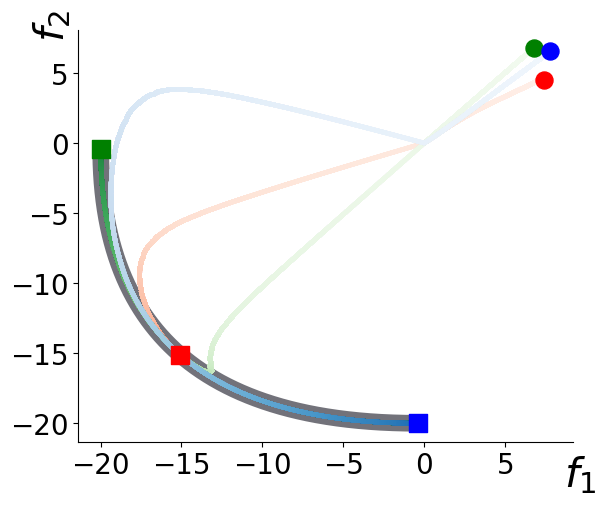}
      \vspace{-.1in}
      \caption{Objective space.}
      \label{fig:toy_pf}
  \end{subfigure}
  \hspace*{10px}
  \begin{subfigure}{0.45\columnwidth}
      \centering
      \includegraphics[width=\textwidth, trim={5, 8, 8, 8}, clip]{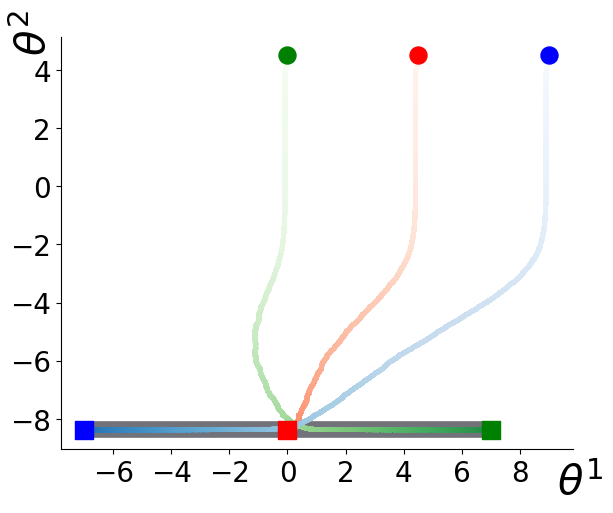}
      \vspace{-.1in}
      \caption{Parameter space.}
      \label{fig:toy_ps}
  \end{subfigure}
  \vspace{-.1in}
     \caption{Trajectory of $\bthe$ obtained by LORPMAN with $\balp = [0.5, 0.5]$ (red), $\balp = [1, 0]$ (green), and $\balp = [0, 1]$ (blue) in objective space (a) and parameter space (b). Circles denote the initial points and squares denote the final points. The gray lines in (a) and (b) denote the PF and PS, respectively.}
     \label{fig:toy}
\end{figure}

\subsection{Datasets with Two or Three Tasks}
\label{sec:two_three}
In this section, we examine the performance on datasets with two or three tasks as used in \cite{PAMAL}, namely, {\it MultiMNIST},
{\it Census},
and {\it UTKFace}.

\begin{figure*}[htbp]
  \centering
  \begin{subfigure}{0.45\textwidth}
      \centering
      \includegraphics[width=\textwidth, trim={8, 8, 8, 25}, clip]{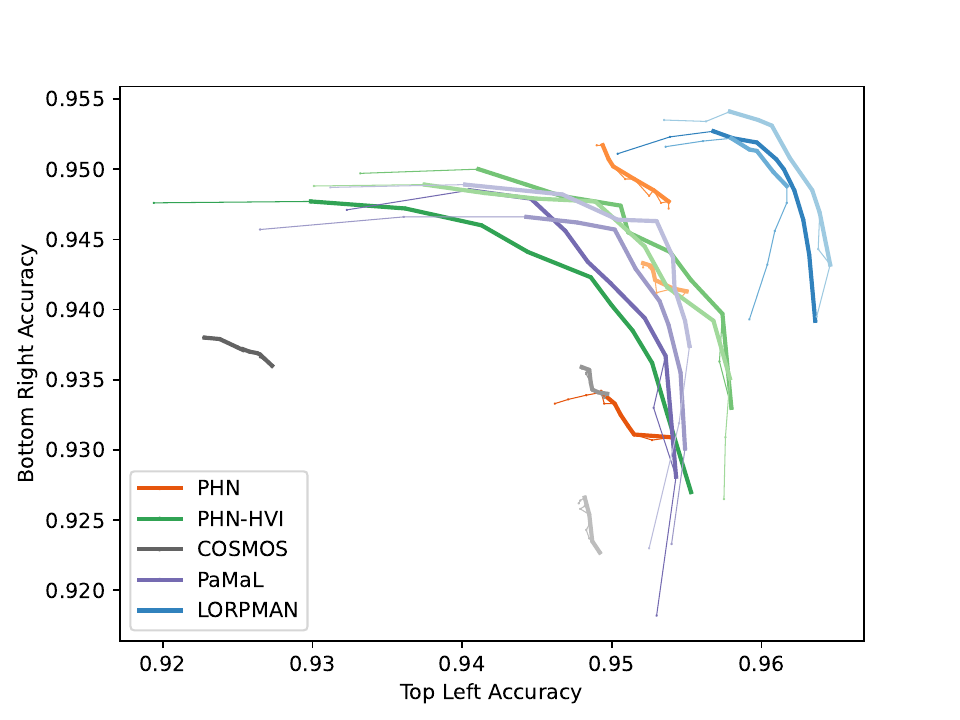}
      \vspace{-.15in}
      \caption{{\it MultiMNIST}.}
      \label{fig:multimnist}
  \end{subfigure}
  \hspace{10px}
  \begin{subfigure}{0.45\textwidth}
      \centering
      \includegraphics[width=\textwidth, trim={8, 8, 8, 25}, clip]{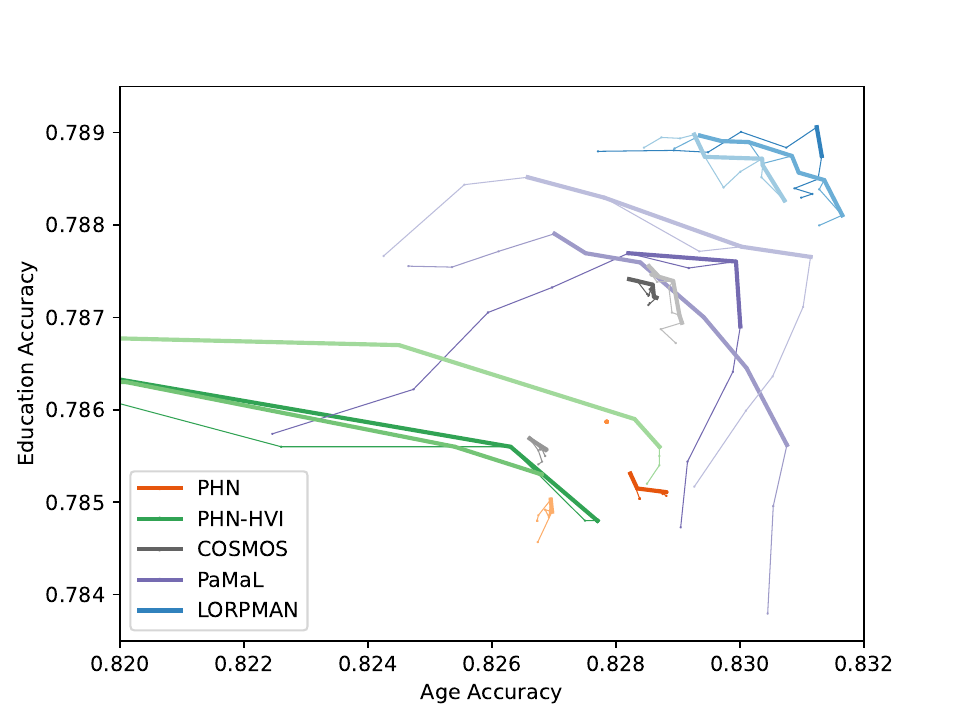}
      \vspace{-.15in}
      \caption{{\it Census}.}
      
      \label{fig:census}
  \end{subfigure}
  \vspace{-.1in}
     \caption{Test performance on {\it MultiMNIST} and {\it Census}. The PF is shown in bold. We show the results obtained by three different random seeds.}
     \label{fig:mnist_census}
\end{figure*}
\begin{table*}[h]
  \centering
  \caption{HV values obtained on \textit{MultiMNIST} and \textit{Census} (averaged over three random seeds). The standard deviation is shown in parentheses.
For PHN and PHN-HVI, we report the number of parameters in the hypernetwork. }
  \begin{tabular}{c|cc|cc} 
  \toprule
        & \multicolumn{2}{c}{\textit{MultiMNIST}} & \multicolumn{2}{c}{\textit{Census}} \\
        & HV & \# Parameters & HV & \# Parameters \\

   \midrule 
   PHN & 0.900 (0.0068) & 2.793M & 0.649 (0.0007) & 12.251M\\
   PHN-HVI & 0.906 (0.0020)& 2.793M & 0.650 (0.0008) & 12.251M\\
   COSMOS & 0.888 (0.0077) & 0.028M & 0.652 (0.0014) & 0.122M \\
   PaMaL & 0.905 (0.0010) &  0.055M & 0.654 (0.0006) & 0.242M\\
   LORPMAN & \bf 0.918 (0.0018) & 0.046M & \bf 0.656 (0.0004) & 0.133M \\

  \bottomrule
  \end{tabular}
  \label{tab:mnist_census}
  \end{table*}

\textbf{MultiMNIST and Census.} {\it MultiMNIST} \cite{multimnist} is a digit classification dataset with two tasks: classification of the top-left digit and classification of the bottom-right digit in each image. {\it Census} \cite{census} is a tabular dataset with two tasks: age prediction and education level classification.

Following \cite{PAMAL}, we use the LeNet\footnote{The parameters of a convolution layer is a four-dimensional tensor instead of a matrix. We adopt the approach in \cite{LoRA} to transform it into a matrix and perform low-rank approximation.} \cite{lenet} and Multilayer Perceptron (MLP) as shared-bottom of the base network for {\it MultiMNIST} and {\it Census}, respectively. For all algorithms, the number of training epochs is set to 10.  
For LORPMAN, we choose the scaling factor $s \in \{1, 2, 4, 6\}$ and freeze epoch $\in \{4, 6, 8\}$
based on the validation set. For both datasets, the rank $r$ for all layers is set to 8 and the orthogonal regularization coefficient $\lambda_o$ is set to 1. We do not tune $r$ and $\lambda_o$. More experimental details on the two datasets
are in Appendices \ref{sec:app_mnist} and \ref{sec:app_census}, respectively.

The resulting PFs
are shown in \cref{fig:mnist_census},
and the corresponding HV values in \cref{tab:mnist_census}. 
As can be seen, LORPMAN can obtain the
PF closer to the top-right region (i.e., with better accuracies on both objectives). COSMOS exhibits limited accuracies due to the constraints imposed by the small number of parameters. Despite PaMaL, PHN, and PHN-HVI having a larger number of parameters,
they show inferior HVs when compared to LORPMAN.

\textbf{UTKFace.}
{\it UTKFace} \cite{utkface} is a dataset with three tasks: age prediction, gender classification and race
classification. Following \cite{PAMAL}, we use the ResNet-18 \cite{resnet} as the shared bottom of the base network. The number of training epochs is 100. We tune
$\lambda_o \in \{0.1, 0.5, 1\}$, freeze epoch $\in \{60, 80\}$, and $r \in \{16, 32, 64\}$. The scaling factor $s$ is set to 1 without tuning. More experimental details are in \cref{sec:app_utkface}.
We do not compare with PHN and PHN-HVI because using the same hypernetwork structure as in \textit{MultiMNIST} and \textit{Census} will result in a hypernetwork with approximately 1 billion parameters.

\cref{tab:utkface} shows the HVs obtained by COSMOS, PaMaL, and LORPMAN and the number of parameters of each algorithm. 
As can be seen, LORPMAN outperforms PaMaL and COSMOS, despite  LORPMAN  has fewer parameters than  PaMaL.
\cref{fig:utkface}
shows the PFs
obtained by PaMaL and LORPMAN.\footnote{Since the PF obtained by COSMOS is inferior, it is omitted from the figure to maintain visual clarity.}
We can see the solutions obtained by LORPMAN outperform PaMaL in all three objectives.

\subsection{Scale to Large Number of Tasks}
\label{sec:large}

\textbf{CIFAR-100.} CIFAR-100 \cite{cifar} is an image classification dataset with 100 classes. These classes are further organized into 20 superclasses. As in \cite{routing, pcgrad, autol}, we consider each superclass as a task. The objective of each task is to accurately classify the image into one of the corresponding 5 more-specific classes.

Following \cite{autol},
we use VGG-16 \cite{vgg} as the base network.
The number of training epochs is set to 300 and we freeze the main model after 250 epochs, which is similar to the proportion in \textit{UTKFace}. The scaling factor $s$ is set to 1 without tuning. We search the orthogonal regularization coefficient $\lambda_o \in \{0.001, 0.005, 0.01, 0.1\}$ and rank $r \in \{8, 16\}$.
More experimental details are in \cref{sec:app_cifar}.

\begin{table}[h]
  \centering
  \caption{HV values obtained on \textit{UTKFace} (averaged over three random seeds). The standard deviation is shown in parentheses.}
  \begin{tabular}{ccc} 
  \toprule
   & HV & \# Parameters \\
   \midrule 
   COSMOS & 0.281 (0.003) & 11.2M \\
   PaMaL &  0.304 (0.003) & 33.6M \\
   LORPMAN &  \bf 0.314 (0.001) & 24.0M \\
  \bottomrule
  \end{tabular}
  \label{tab:utkface}
  \end{table}

\begin{figure}[h]
  \centering
      \begin{subfigure}{0.48\columnwidth}
      \centering
      \includegraphics[width=\textwidth]{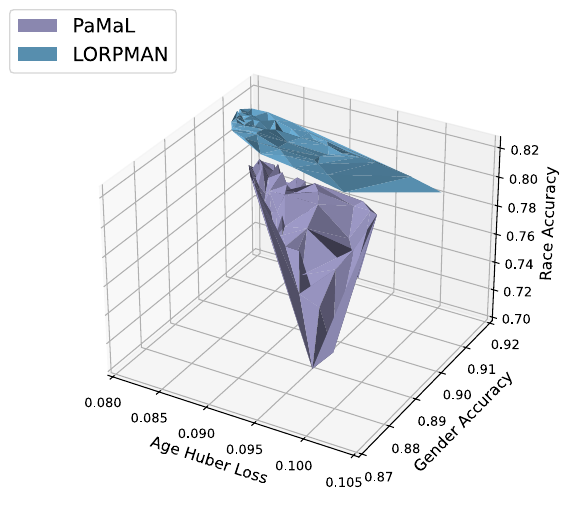}
      \label{fig:utk_a}
      \vspace*{-20px}
      \caption{}
  \end{subfigure}
  \begin{subfigure}{0.48\columnwidth}
      \centering
      \includegraphics[width=\textwidth]{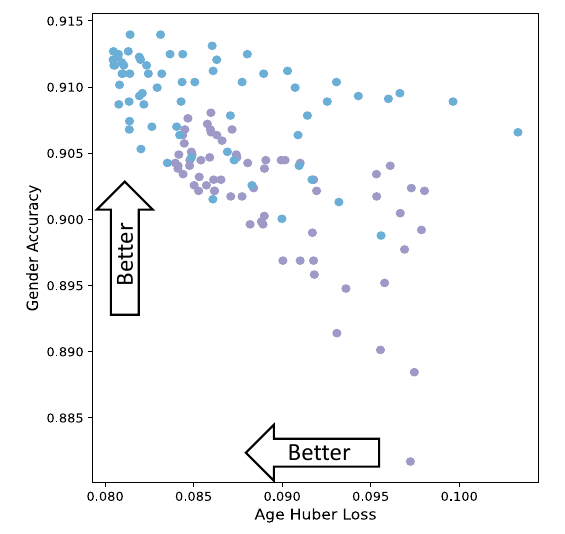}
      \label{fig:utk_b}
      \vspace*{-20px}
      \caption{}
  \end{subfigure}
  \begin{subfigure}{0.48\columnwidth}
      \centering
      \includegraphics[width=\textwidth]{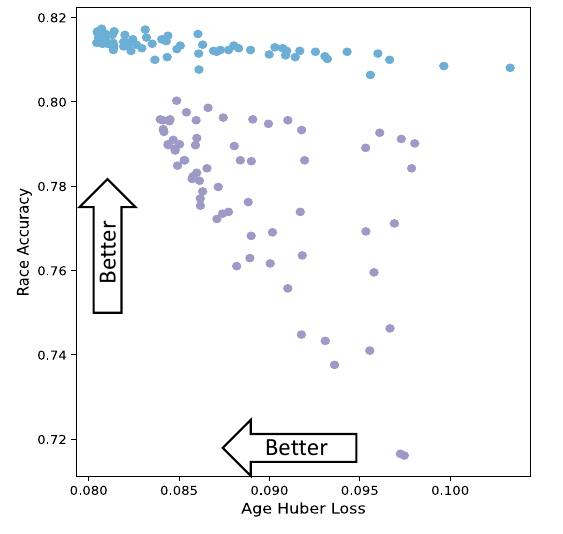}
      \label{fig:utk_c}
      \vspace*{-20px}
      \caption{}
  \end{subfigure}
  \begin{subfigure}{0.48\columnwidth}
      \centering
      \includegraphics[width=\textwidth]{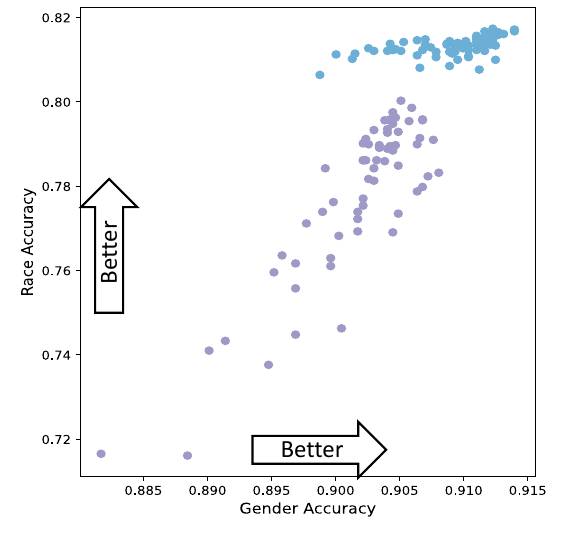}
      \label{fig:utk_d}
      \vspace*{-20px}
      \caption{}
  \end{subfigure}
  \vspace{-.1in}
      \caption{Test performance of PaMaL and LORPMAN on {\it UTKFace}. Figures~(b), (c), (d) are 2D projections of (a) for better illustration of the 3D surface.}
      \label{fig:utkface}
\end{figure}

The obtained HVs are shown in \cref{tab:cifar100}. The obtained PFs are in \cref{sec:pf_cifar}.
The results highlight the challenges encountered by PaMaL when dealing with a large number of objectives. PaMaL needs to jointly train 20 base networks, which leads to a large number of parameters and a small hypervolume. In contrast, LORPMAN achieves significantly better hypervolume while utilizing only 8.9\% of the parameters compared to PaMaL. This underscores the efficiency and effectiveness of LORPMAN in addressing problems with a large number of tasks.

\begin{table}[h]
  \centering
  \caption{HV values obtained on \textit{CIFAR-100} (averaged over three random seeds). The standard deviation is shown in parentheses.
  }
  \begin{tabular}{ccc} 
  \toprule
   & HV ($\times 10^{-2}$) & \# Parameters \\
   \midrule 
   COSMOS & 0.344 (0.018)  & 15.0M \\
   PaMaL &  0.0583 (0.010) & 296.5M \\
   LORPMAN &  \bf 0.887 (0.047) & 26.4M \\
  \bottomrule
  \end{tabular}
  \label{tab:cifar100}
  \end{table}

\cref{fig:loss} shows the convergence of the training losses of PaMaL with different learning rates. The training loss of LORPMAN 
(with learning rate $0.01$)
is also shown for comparison. 
We can observe that the poor performance of PaMaL on {\it CIFAR-100} is due to its slow convergence, since it has to jointly train 20 base networks. Simply increasing the learning rate does not help the convergence of PaMaL.

\begin{figure}[htbp]
    \centering
      \includegraphics[width=0.9\columnwidth, trim={8, 7, 8, 40}, clip]{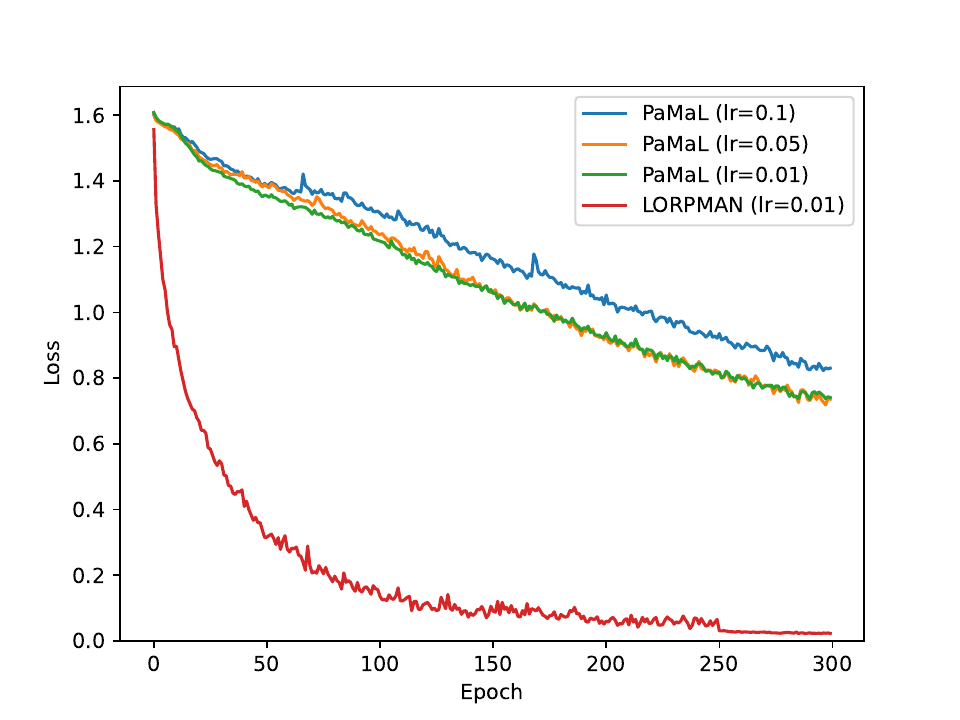}
     \vspace{-.1in}
     \caption{Training loss of PaMaL with number of epochs, using different learning rates.}
     \label{fig:loss}     
\end{figure}

\textbf{CelebA.} \textit{CelebA} \cite{celeba} is a face attribute classification dataset with 40 tasks, where each task is a binary classification of a face attribute. 
Following \cite{mtlmo}, we use ResNet-18 \cite{resnet} as the base network. The number of training epochs is 50, and we freeze the main model after 40 epochs, which is similar to the proportion in \textit{UTKFace}. The scaling factor $s$ is set to 1 without tuning. We search the orthogonal regularization coefficient $\lambda_o \in \{0.1, 0.5, 1\}$ and rank $r \in \{8, 16, 32\}$. Here, we consider two more baselines: (i) PHN with chunking \cite{PHNarxiv, cpmtl},
both with the original setting in \cite{PHNarxiv} (i.e, the hidden dimension is set to 100)
and a scaled setting with comparable number of parameters as LORPMAN (i.e., the hidden dimension is set to 500);
and (ii) FiLM condition, with a FiLM condition layer added after each block of ResNet-18. 
More experimental details are described in \cref{sec:app_celeba}. 

The HV values are shown in \cref{tab:celeba}. We can see that LORPMAN achieves much better performance while using only 21\% of parameters compared to PaMaL. The 
FiLM condition suffers similar problems as COSMOS due to the limited number of parameters. PHN with chunking also shows worse performance than LORPMAN. In comparison, the proposed LORPMAN is a more straightforward approach that achieves good performance and parameter efficiency.

  \begin{table}[h]
    \centering
    \caption{HV values obtained on \textit{CelebA}  (averaged over three random seeds). The standard deviation is shown in parentheses.
    }
    \begin{tabular}{ccc} 
    \toprule
     Method & HV ($\times 10^{-2}$) & \# Parameters \\
     \midrule
     PHN-Chunking (original) &  0.663 (0.027) & 36.6M \\
     PHN-Chunking (scaled) & 0.681 (0.048)& 92.3M \\
     FiLM   &  0.803 (0.022) & 11.4M \\
     COSMOS & 0.783 (0.013)  & 11.4M \\
     PaMaL &  0.472 (0.018) & 453.3M \\
     LORPMAN  &  \bf 1.167 (0.008)  & 96.8M \\

    \bottomrule
    \end{tabular}
    \label{tab:celeba}
    \end{table}

\subsection{Comparison with 
Discrete PF 
Baselines}
\label{sec:discrete_baseline}

In this experiment, we compare with five algorithms that obtain discrete PFs: EPO \cite{EPO}, PMTL \cite{PMTL}, MOO-SVGD \cite{MOOSVGD}, PNG \cite{PNG}, and GMOOAR \cite{gmooar}. Since most of these algorithms have performed experiments on \textit{MultiMNIST}, we use \textit{MultiMNIST} for comparison. The settings are the same as in \cref{sec:two_three}. 

The obtained solutions are shown in \cref{fig:mnist_dis}.
As can be seen, the proposed algorithm has better performance than existing discrete PF algorithms, even though they can
only generate a predetermined number of solutions.

\begin{figure}[t]
  \centering
      \includegraphics[width=0.9\columnwidth, trim={8, 7, 8, 35}, clip]{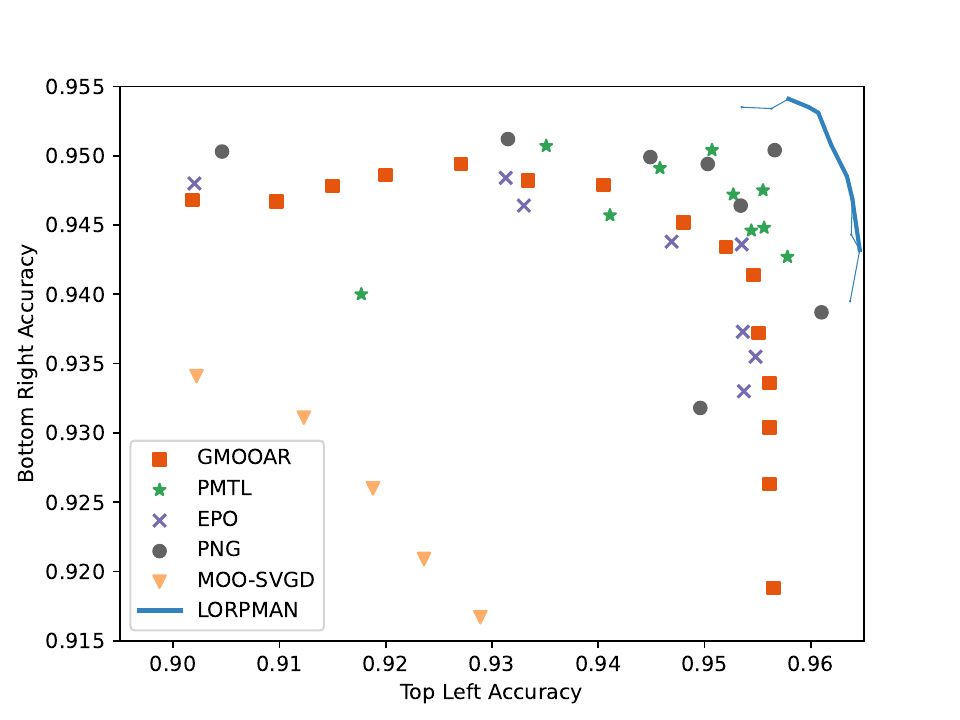}
      \vspace{-.1in}
      \caption{Test performance of LORPMAN and various discrete PF algorithms on {\it MultiMNIST}.}
      \vspace{-.1in}
      \label{fig:mnist_dis}
\end{figure}

\subsection{Comparison with Single-Solution Baselines}
\label{sec:single_baseline}
In this experiment, we compare with multi-task learning baselines 
that can only generate a single solution. These include
Linear Scalarization (LS), UW \cite{uw}, MGDA \cite{mtlmo}, DWA \cite{dwa}, PCGrad \cite{pcgrad}, IMTL \cite{imtl}, Graddrop \cite{graddrop}, CAGrad \cite{CAGrad}, RLW \cite{rlw}, Nash-MTL \cite{NashMTL}, RotoGrad \cite{rotograd}, and Auto-$\lambda$ \cite{autol}. For a more complete comparison, we also include the results obtained by Single-Task Learning (STL), where each task is trained independently, and the baselines of COSMOS and PaMaL which provide continuous approximations of the PF.

Following \cite{PAMAL}, we use a widely-used dataset \textit{CityScapes} \cite{cityscapes}. It is a scene-understanding dataset with two tasks: semantic segmentation and depth regression. We adopt the SegNet \cite{segnet} as the base network and use the same parameter configuration as in \cite{PAMAL}. The number of training epochs is 100. We tune freeze epoch $\in \{80, 90\}$, $s \in \{0.5, 1\}$, and $r \in \{32, 64\}$. More experimental details are in \cref{sec:app_city}.

LORPMAN/COSMOS/PaMaL can provide the whole PF.
We only select one solution from their obtained PFs for comparison
with the single-solution algorithms.
\cref{tab:cityscapes}
shows the
performance on 
semantic segmentation 
(mIOU and pixel accuracy)
and 
depth prediction 
(absolute and relative error)
on the test set.
We can see that LORPMAN not only surpasses COSMOS and PaMaL but also outperforms the single-solution algorithms. While MGDA achieves slightly better relative error in depth estimation, it has much worse performance in the semantic segmentation task. 

It is important to note that our objective is to learn the entire PF rather than focusing on a single solution.

\begin{table}[t]
  \footnotesize
  \renewcommand{\arraystretch}{0.95}
  \centering
  \caption{Test performance on \textit{CityScapes}. We show the average over three random seeds. Results, excluding LORPMAN, are taken from \cite{PAMAL}}. 
  \begin{tabular}{@{}lcccccccc@{}}
    \toprule
      & \multicolumn{2}{c}{Segmentation} &                      & \multicolumn{2}{c}{Depth} \\
      \cmidrule(lr){2-3} \cmidrule(lr){5-6}
       & \multicolumn{1}{l}{mIoU $\uparrow$} & \multicolumn{1}{l}{Pix Acc $\uparrow$} & \multicolumn{1}{l}{} & \multicolumn{1}{l}{Abs Err $\downarrow$} & \multicolumn{1}{l}{Rel Err $\downarrow$} \\
    \midrule
    STL                 &      70.96 &        92.12 &                      &         0.0141 &     \bf 38.644 \\
    \midrule
    LS                  &      70.12 &        91.90 &                      &         0.0192 &        124.061 \\
    UW                  &      70.20 &        91.93 &                      &         0.0189 &        125.943 \\
    MGDA                &      66.45 &        90.79 &                      &         0.0141 &         53.138 \\
    DWA                 &      70.10 &        91.89 &                      &         0.0192 &        127.659 \\
    PCGrad              &      70.02 &        91.84 &                      &         0.0188 &        126.255 \\
    IMTL                &      70.77 &        92.12 &                      &         0.0151 &         74.230 \\
    Graddrop            &      70.07 &        91.93 &                      &         0.0189 &        127.146 \\
    CAGrad              &      69.23 &        91.61 &                      &         0.0168 &        110.139 \\
    RLW                 &      68.79 &        91.52 &                      &         0.0213 &        126.942 \\
    Nash-MTL            &      71.13 &        92.23 &                      &         0.0157 &         78.499 \\
    RotoGrad            &      69.92 &        91.85 &                      &         0.0193 &        127.281 \\
    Auto-$\lambda$      &      70.47 &        92.01 &                      &         0.0177 &        116.959 \\
    \midrule
    COSMOS              &      69.78 &        91.79 &                      &         0.0539 &        136.614 \\
    PaMaL               &      70.35 &        91.99 &                      &         0.0141 &         54.520 \\
    LORPMAN    &   \bf 72.13 &   \bf 92.57 &                     &        \bf  0.0135 &         54.942 \\
    \bottomrule
    \end{tabular}
    
  \label{tab:cityscapes}
  \vspace*{-10px}
  \end{table}

\subsection{Ablation Studies}
\label{sec:ablation}
In this section, we investigate the effects of (i) orthogonal regularization, (ii) rank $r$, (iii) freeze epoch, and (iv) scaling factor $s$ on the performance of LORPMAN. Experiment is performed on the \textit{UTKFace} dataset.

\textbf{Orthogonal Regularization.}
\cref{tab:ablation2} shows the impact of orthogonal regularization on HV and the average correlation\footnote{The correlation between a pair of low-rank matrices 
$\mathbf{B}_i\mathbf{A}_i$ and
$\mathbf{B}_j\mathbf{A}_j$ 
is computed as $\texttt{flatten}(\mathbf{B}_i\mathbf{A}_i)^{\top}\texttt{flatten}(\mathbf{B}_j\mathbf{A}_j) / (\norm{\mathbf{B}_i\mathbf{A}_i}\norm{\mathbf{B}_j\mathbf{A}_j})$.} over all pairs of low-rank matrices.
We can see that with orthogonal regularization, the correlation between low-rank matrices is significantly reduced. Similar observations can also be found in traditional orthogonal regularization for parameters within a single neural network \cite{othreg1, orthogonality}. Such reduction in correlation encourages learning common features in the main network and differences in the low-rank matrices, thus leading to better HV value.

\begin{table}[ht]
    \centering
    \caption{Effects of orthogonal regularization on \textit{UTKFace}  (averaged over six random seeds).}
   \centering
\begin{tabular}{ccc}
    \toprule
     & HV & Correlation\\
    \midrule 
    w/o orthogonal reg &   0.309 (0.002) & 0.466 (0.045)\\
    w/ orthogonal reg & \bf 0.314 (0.001) & 0.067 (0.012)\\
   \bottomrule
  \end{tabular}
    \label{tab:ablation2}
\end{table}

\textbf{Rank $r$.} 
\cref{tab:ablation1} shows the effects of the rank $r$ on the HV and number of parameters. We can observe that LORPMAN with a rank $8$ can already outperform COSMOS and PaMaL (see \cref{tab:utkface}). By increasing the rank to 64, even better performance can be achieved.  However, further increasing the rank does not yield significant change.

\begin{table}[h]
  \centering
  \caption{Effects of rank $r$ on \textit{UTKFace} (averaged over six random seeds).}
  \begin{tabular}{ccc} 
  \toprule
   $r$ & HV & \# parameters \\
   \midrule 
   4 & 0.306 (0.002) & 12.1M\\
   8 & 0.310 (0.001) & 12.8M\\
   16 & 0.310 (0.002) & 14.5M \\
   32 & 0.311 (0.001) & 17.6M \\
   64 &  \bf 0.314 (0.001) & 24.0M \\
   128 &  0.313 (0.002) & 35.5M \\
  \bottomrule
  \end{tabular}
  \label{tab:ablation1}
  \end{table}

\textbf{Freeze Epoch.}
\cref{tab:ablation3} shows the effect of freeze epoch on HV. 
Compared with not freezing the main module (i.e., freeze epoch = 100),
freezing during the latter half of the training process 
encourages the low-rank matrices to learn task-specific representations instead of always relying on the main model, thus leading to better performance.

\textbf{Scaling Factor $s$.} \cref{tab:ablation4} shows the effect of scaling factor $s$ on HV. As can be seen, setting $s$ within a reasonable range (such as $[0.1, 1]$) leads to stable performance.

\begin{table}[htbp]
  \centering

\begin{minipage}{0.48\linewidth} 
\caption{Effects of freeze epoch on \textit{UTKFace} (averaged over six random seeds).}
\begin{tabular}{cc}
    \toprule
    freeze & HV \\
    \midrule 
    20 & 0.305 (0.001) \\
    40 & 0.311 (0.001) \\
    60 & 0.312 (0.002) \\
    80 & \bf 0.314 (0.001)  \\
    100 & 0.307 (0.003) \\
   \bottomrule
  \end{tabular}
    \label{tab:ablation3}
\end{minipage}
\hfill
\begin{minipage}{0.48\linewidth}
\caption{Effects of scaling factor $s$ on \textit{UTKFace} (averaged over six random seeds).}
\begin{tabular}{cc}
    \toprule
    $s$ & HV \\
    \midrule 
     0.1&  0.313 (0.001) \\
     0.5&  0.313 (0.001) \\
     1&  \bf 0.314 (0.001) \\
     1.5 & 0.312 (0.002) \\
     2&   0.310 (0.002) \\
   \bottomrule
  \end{tabular}
    \label{tab:ablation4}
\end{minipage}
  
\end{table}

\section{Conclusion}
In this paper, we introduce a novel method for continuous approximation of the PF. We use a main network with a collection of low-rank matrices. Our approach leverages the inherent structure of Pareto manifold to effectively learn trade-off solutions between tasks. Extensive empirical evaluation on various datasets demonstrates the superior performance of the proposed algorithm, especially when the number of tasks is large. 

One limitation of the proposed work is that we consider the same rank for all layers. Using different ranks for different layers to achieve better parameter efficiency and exploring automatic rank setting strategies could be interesting future directions. 

\section*{Acknowledgements}
This research was supported in part by the Research Grants Council of the Hong Kong Special
Administrative Region (Grant 16202523).

\section*{Impact Statement}
The proposed algorithm can effectively discover trade-off solutions among multiple tasks, leading to improved performance and cost efficiency in real-world applications involving multiple objectives. However, we should be aware the potential biases and privacy concerns when processing data from different tasks. 


\bibliography{reference.bib}
\bibliographystyle{icml2024}

\newpage
\appendix
\onecolumn
\section{Proof of \cref{theo:approx}}
\label{sec:proof}
Denote the optimal mapping from network input $\bx \in X \subset \R^u$ and preference vector $\balp \in \Delta^m$ to the corresponding point on the PF as $t(\bx, \balp): X \times \Delta^m \rightarrow \R^m$.
We have the following Theorem.
\begin{theorem}
Assume that $X \times \Delta^m$ is compact and $t(\bx, \balp)$ is continuous. For any $\epsilon > 0$, there exists a ReLU MLP $h$ with main network $\bthe_0$ and $m$ low-rank matrices $\bB_1\bA_1, \ldots, \bB_m\bA_m$, such that $\forall \bx \in X, \forall \balp \in \Delta^m$,
\begin{align*}
  \norm{t(\bx, \balp)-h\left(\bx; \bthe_0 + \sum_{i=1}^{m}\alpha_i \bB_i\bA_i\right)}\leq \epsilon.
\end{align*}
\end{theorem}

\begin{proof}
Denote $\sigma(\bx) \equiv \max(\bm{0}, \bx)$. From the universal approximation theorem \cite{uat}, for any $\epsilon > 0$, there exists $v \in \mathbb{N}, \bM \in \R^{v \times (u+m)}, \bN \in \R^v, \bC \in \mathbb{R}^{m \times v}$ such that
\begin{align*}
  \sup_{\bx \in X, \alpha \in \Delta^m} \norm{t(\bx, \balp)-g(\bx, \balp)}\leq \epsilon,
\end{align*}
where $g(\bx, \balp) = \bC \sigma(\bM[\bx,\balp]^\top + \bN)$.

We define two matrices $\bR \in \R^{u \times (2u+m)}$ and $\bS \in \R^{(2u+m) \times u}$ as follows:
\begin{align*}
  \bR_{i, j} = 
  \begin{cases}
    1,  & \text{if } j = 2i-1\\
    -1, & \text{if } j = 2i \\
    0, & \text{otherwise}
  \end{cases},
\end{align*}
and
\begin{align*}
  \bS_{i, j} = 
  \begin{cases}
    1,  & \text{if } i = 2j-1 \text{ or } (i>2u \text{ and } j=u)\\
    -1, & \text{if } i = 2j \\ 
    0, & \text{otherwise} 
  \end{cases}.
\end{align*}
Define $m$ vectors $\bU_1, \ldots, \bU_m \in \{0,1\}^{2u+m}$ as:
\begin{align*}
  (\bU_i)_j = 
  \begin{cases}
    1, & j = 2m+i \\
    0, & \text{otherwise} 
  \end{cases}.
\end{align*}
Then, we have
\begin{align*}
  \bS \sigma\left(\bR \bx + \sum_{i=1}^{m} \alpha_i \bU_i \right) = [\bx, \balp].
\end{align*}

We can construct a MLP $h(\bx; \bM, \bN, \bC, \bR, \bS, \bU) = \bC \sigma(\bM\bS\sigma(\bR \bx + \bU) + \bN)$. 

Let $\bthe_0 = (\bM, \bN, \bC, \bR, \bS, \bm{0})$ and $\bthe_i = (\bm{0}, \bm{0}, \bm{0},\bm{0}, \bm{0}, \bU_i), i=1, \ldots, m$. We have 
  \begin{align*}
  h(x;\bthe_0 + \sum_{i=1}^{m}\alpha_i \bthe_i) & = h(\bm{x}; \bM, \bN, \bC, \bR, \bS, \sum_{i=1}^{m}\alpha_i \bU_i)           \\
                                   & = \bC \sigma(\bM\bS\sigma(\bR \bx + \sum_{i=1}^{m} \alpha_i \bU_i) + \bN) \\
                                   & = g(\bS\sigma(\bR \bx + \sum_{i=1}^{m} \alpha_i \bU_i))                \\
                                   & = g(\bx, \balp).
  \end{align*}
Since $\bthe_i$ consists of a single non-zero entry (equal to 1) with all other elements being zero, it can be reshaped into a matrix $\bB_i\bA_i$ with rank 1.
  \end{proof}

\section{Proof of \cref{prop:po}}
\label{sec:proof_po}
\begin{proposition}[\cite{convex}]
Given any $\balp \in \{\balp ~|~ \alpha_i > 0\}$, the optimal solution of the scalarized problem $\min_{\bthe} \sum_{i=1}^m \alpha_i f_i(\bthe)$ is a Pareto-optimal solution of the original multi-objective optimization problem \cref{eq:moo}.   
\end{proposition}
\begin{proof}
Suppose $\bthe$ is the optimal solution of $\sum_{i=1}^m \alpha_i f_i(\bthe)$ but it is not Pareto optimal. Based on the definition of Pareto optimal, there exists $ \bthe'$, such that for $i \in [m]$, $f_i(\bthe) \geq f_i(\bthe')$, and $\exists i \in [m], f_i(\bthe) > f_i(\bthe')$. Since $\alpha_i > 0$, we have
\begin{equation*}
    \sum_{i=1}^m \alpha_i (f_i(\bthe)-f_i(\bthe')) > 0.
\end{equation*}
It can be rewritten as:
\begin{equation*}
    \sum_{i=1}^m \alpha_i f_i(\bthe) > \sum_{i=1}^m \alpha_i f_i(\bthe'),
\end{equation*}
which contradicts the assumption that $\bthe$ is the optimal solution of $\sum_{i=1}^m \alpha_i f_i(\bthe)$.
\end{proof}
\section{Multi-Forward Regularization}
\label{sec:mf_reg}
\citet{PAMAL} propose the multi-forward regularization to penalize the wrong ordering of solutions. Denote the set of current sampled reference vectors as $\mathcal{V} \equiv \{\balp^1, \ldots, \balp^b\}$. Then, a  directed graph $\mathcal{G}_i = (\mathcal{V}, \mathcal{E}_i)$ is constructed for each task where $\mathcal{E}_i=\{(\balp, \balp')\in \mathcal{V} \times \mathcal{V} : \alpha_i<\alpha'_i \}$. The multi-forward regularization loss is defined as:
\begin{align*}
  \mathcal{R}_p = \sum_{i=1}^{m}\log\left(\frac{1}{|\mathcal{E}_i|}\sum_{(\balp, \balp')\in \mathcal{E}_i} e^{[f(\bthe(\balp)) -f(\bthe(\balp'))]_+}\right).
\end{align*}

\section{Experimental Details}
\label{sec:exp_details}
\subsection{Hypervolume}
\label{sec:hv}
Hypervolume (HV) \cite{hv1, hv2} is a popular indicator in multi-objective optimization (MOO), offering a measure of the performance of the obtained solution set. Formally, for a given solution set $\mathcal{P}$ and a reference point $\br$, the HV is defined as:
\begin{align}
  HV(\mathcal{P}; \br) = \Lambda (\{\bq \in \R^m ~~|~~ \exists \bp \in \mathcal{P} : \bp \geq \bq \geq \br\}),
\end{align}
where $\Lambda$ denotes the Lebesgue measure.
\cref{fig:hv} provides an illustration of the HV indicator in a two-objective maximization 
problem. The shaded area encapsulated by the solution set and the reference point represents the HV.
\begin{figure}[h]
  \centering
      \includegraphics[width=0.4\columnwidth]{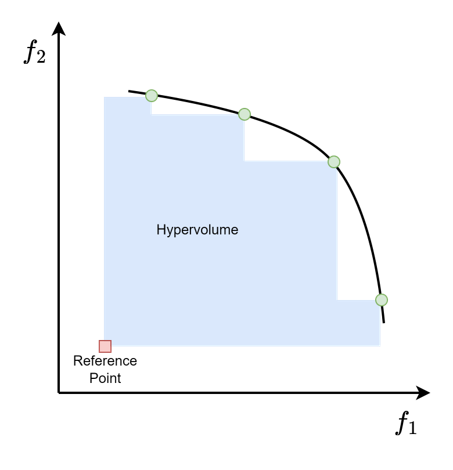}
      \caption{Illustration of the hypervolume indicator.}
      \label{fig:hv}
\end{figure}

For all datasets except \textit{UTKFace} \cite{utkface}, the reference point for HV evaluation is set to $[0, 0, \ldots, 0]$, since $0$ is the smallest-possible accuracy. For \textit{UTKFace}, the first objective is the Huber loss \cite{huber}, not accuracy. Since $0.5$ is close to the observed largest Huber loss in the experiment, we set the reference point to $[0.5, 0, 0]$.

\subsection{Toy Problem}
\label{sec:toy}
We use the two-objective toy problem adopted by \cite{CAGrad, NashMTL, PAMAL}, with parameters $\bthe = [\theta^1, \theta^2] \in \R^2$. The problem is formulated as follows:
\begin{flalign*}
  \text{Minimize}~ f_1(\bthe) &= c_1(\theta)h_1(\theta) + c_2(\theta)g_1(\theta) ~~\text{and}~~
   f_2(\bthe) = c_1(\theta)h_2(\theta) + c_2(\theta)g_2(\theta),\\
  ~\text{where}~~
  h_1(\bthe) &= \log{\big(\max(|0.5(-\theta^1-7)-\tanh{(-\theta^2)}|,~~0.000005)\big)} + 6, \\
  h_2(\bthe) &= \log{\big(\max(|0.5(-\theta^1+3)-\tanh{(-\theta^2)}+2|,~~0.000005)\big)} + 6, \\
  g_1(\bthe) &= \big((-\theta^1+7)^2 + 0.1*(-\theta^2-8)^2\big)/10-20, \\
  g_2(\bthe) &= \big((-\theta^1-7)^2 + 0.1*(-\theta^2-8)^2)\big/10-20, \\
  c_1(\bthe) &= \max(\tanh{(0.5*\theta^2)},~0)~~\text{and}~~c_2(\bthe) = \max(\tanh{(-0.5*\theta^2)},~0).
\end{flalign*}

\subsection{MultiMNIST}
\label{sec:app_mnist}
\textbf{Network Structure.} Following \cite{PAMAL}, the base network consists of a shared bottom and two task-specific heads. The shared bottom is the LeNet \cite{lenet} with two convolution layers and a fully-connected layer. Each task-specific head consists of two full-connect layers. For PHN and PHN-HVI, we use the same hypernetwork structure as \cite{PHN}. The input preference vector $\balp$ first goes through a three-layer MLP to get the ray embedding. Then, for each layer of the base network, a linear layer is used to generate the base network parameters based on the ray embedding.

\textbf{Parameter Settings.} Following \cite{PAMAL}, for LORPMAN and PaMaL, we set the multi-forward regularization coefficient $\lambda_p$ to 0, and window size $b$ to 4.
The distribution parameters $p$ for sampling $\balp$ is set to 1 for all algorithms. For PHN and PHN-HVI, we use the same parameter setting as in the original paper. We use the Adam optimizer. For COSMOS, the cosine similarity regularization coefficient $\lambda_c$ is set to 1. The learning rate is set to $1e-3$ and the batch size is set to 256.

\subsection{Census}
\label{sec:app_census}
\textbf{Network Structure.} Following \cite{PAMAL}, the base network consists of a shared bottom and two task-specific heads. The shared bottom is an one-layer MLP. Each task-specific head is a fully-connected layer. The hypernetwork is constructed in the same way as for \textit{MultiMNIST}. 

\textbf{Parameter Settings.} Following \cite{PAMAL}, for LORPMAN and PaMaL, we set the multi-forward regularization coefficient $\lambda_p$ to 5, and window size $b$ to 2. The distribution parameters $p$ for sampling $\balp$ is set to 1 for all algorithms. For PHN-HVI, we use the same parameter setting as in the original paper. For COSMOS, the cosine similarity regularization coefficient $\lambda_c$ is set to 1. We use the Adam optimizer. The learning rate is set to $1e-3$ and the batch size is set to 256.
\subsection{UTKFace}
\label{sec:app_utkface}
\textbf{Network Structure.} Following \cite{PAMAL}, the base network consists of a shared bottom and three task-specific heads. The shared bottom is a ResNet-18 \cite{resnet}. Each task-specific head is a fully-connected layer.

\textbf{Parameter Settings.} Following \cite{PAMAL}, for LORPMAN and PaMaL, we set the multi-forward regularization coefficient $\lambda_p$ to 1, and window size $b$ to 3. The distribution parameters $p$ for sampling $\balp$ is set to 2 for all algorithms. For COSMOS, the cosine similarity regularization coefficient $\lambda_c$ is set to 1. We use the Adam optimizer. The learning rate is set to $1e-3$ and the batch size is set to 256.

\subsection{CIFAR-100}
\textit{CIFAR-100} has 20 5-way classification tasks: \\
1. Aquatic Mammals: beaver, dolphin, otter, seal, whale \\
2. Fish: aquarium fish, flatfish, ray, shark, trout \\
3. Flowers: orchids, poppies, roses, sunflowers, tulips \\
4. Food Containers: bottles, bowls, cans, cups, plates \\
5. Fruits and Vegetables: apples, mushrooms, oranges, pears, sweet peppers \\
6. Household Electrical Devices: clock, computer keyboard, lamp, telephone, television \\
7. Household Furniture: bed, chair, couch, table, wardrobe \\
8. Insects: bee, beetle, butterfly, caterpillar, cockroach \\
9. Large Carnivores: bear, leopard, lion, tiger, wolf \\
10. Large Man-made Outdoor Things: bridge, castle, house, road, skyscraper \\
11. Large Natural Outdoor Scenes: cloud, forest, mountain, plain, sea \\
12. Large Omnivores and Herbivores: camel, cattle, chimpanzee, elephant, kangaroo \\
13. Medium-sized Mammals: fox, porcupine, possum, raccoon, skunk \\
14. Non-insect Invertebrates: crab, lobster, snail, spider, worm \\
15. People: baby, boy, girl, man, woman \\
16. Reptiles: crocodile, dinosaur, lizard, snake, turtle \\
17. Small Mammals: hamster, mouse, rabbit, shrew, squirrel \\
18. Trees: maple, oak, palm, pine, willow \\
19. Vehicles 1: bicycle, bus, motorcycle, pickup truck, train \\
20. Vehicles 2: lawn-mower, rocket, streetcar, tank, tractor \\

\label{sec:app_cifar}
\textbf{Network Structure.} Following \cite{autol}, the base network consists of a shared bottom and 20 task-specific heads. The shared bottom is a VGG-16 \cite{vgg}. Each task-specific head is a fully-connected layer. 

\textbf{Parameter Settings.} For LORPMAN and PaMaL, we set the multi-forward regularization coefficient $\lambda_p$ to 1, and window size $b$ to 2. The distribution parameters $p$ for sampling $\balp$ is set to 2 for all algorithms. For COSMOS, the cosine similarity regularization coefficient $\lambda_c$ is set to 1. Following \cite{autol}, we use the SGD optimizer. The learning rate is set to $1e-2$ and the batch size is set to 64. 
\subsection{CelebA} 
\label{sec:app_celeba}
CelebA \cite{celeba} has 40 binary classification tasks: 5 o'Clock Shadow, Arched Eyebrows, Attractive, Bags Under Eyes, Bald, Bangs, Big Lips, Big Nose, Black Hair, Blond Hair, Blurry, Brown Hair, Bushy Eyebrows, Chubby,
Double Chin, Eyeglasses, Goatee, Gray Hair, Heavy Makeup, High Cheekbones, Male, Mouth Slightly Open, Mustache, Narrow Eyes, No Beard, Oval Face, Pale Skin, Pointy Nose, Receding Hairline, Rosy Cheeks, Sideburns, Smiling, Straight Hair, Wavy Hair, Wearing Earrings, Wearing Hat, Wearing Lipstick, Wearing Necklace, Wearing Necktie, Young.

\textbf{Network Structure.} Following \cite{mtlmo}, the base network consists of a shared bottom and 40 task-specific heads. The shared bottom is a ResNet-18 \cite{resnet}. Each task-specific head is a fully-connected layer. Since CelebA has the overfitting problem \cite{unitary}, we adopt dropout regularization as in \cite{unitary} for all algorithms (i.e., we add a dropout layer in each block as well as after the first convolution layer and after average pooling layer).

\textbf{Parameter Settings.} For LORPMAN and PaMaL, we set the multi-forward regularization coefficient $\lambda_p$ to 1, and window size $b$ to 3. The distribution parameters $p$ for sampling $\balp$ is set to 2 for all algorithms. For COSMOS, the cosine similarity regularization coefficient $\lambda_c$ is set to 1. We use the Adam optimizer. The learning rate is set to $1e-3$ and the batch size is set to 128.

\subsection{Cityscapes} 
\label{sec:app_city}
\textbf{Network Structure.} Following \cite{CAGrad, NashMTL, PAMAL}, the base network consists of a shared bottom and 2 task-specific heads. The shared bottom is a SegNet \cite{segnet}. Each task-specific head consists of two convolution layers. 

\textbf{Parameter Settings.} Following \cite{PAMAL}, for LORPMAN and PaMaL, we set the multi-forward regularization coefficient $\lambda_p$ to 1, and window size $b$ to 3. The distribution parameters $p$ for sampling $\balp$ is set to 7 for all algorithms. We use the Adam optimizer. The initial learning rate is set to $1e-4$ and is halved after 75 epochs. The batch size is set to 8.

\section{Pareto Front Obtained on CIFAR-100}
\label{sec:pf_cifar}

\cref{fig:cifar100} shows the models obtained by COSMOS, PaMaL and LORPMAN. We can observe that the solutions obtained by PaMaL have much lower accuracies while the solutions obtained by COSMOS concentrate in a small region. The proposed LORPMAN achieves good accuracy and diversity.
\begin{figure}[b]
    \centering
    \includegraphics[width=0.65\textwidth]{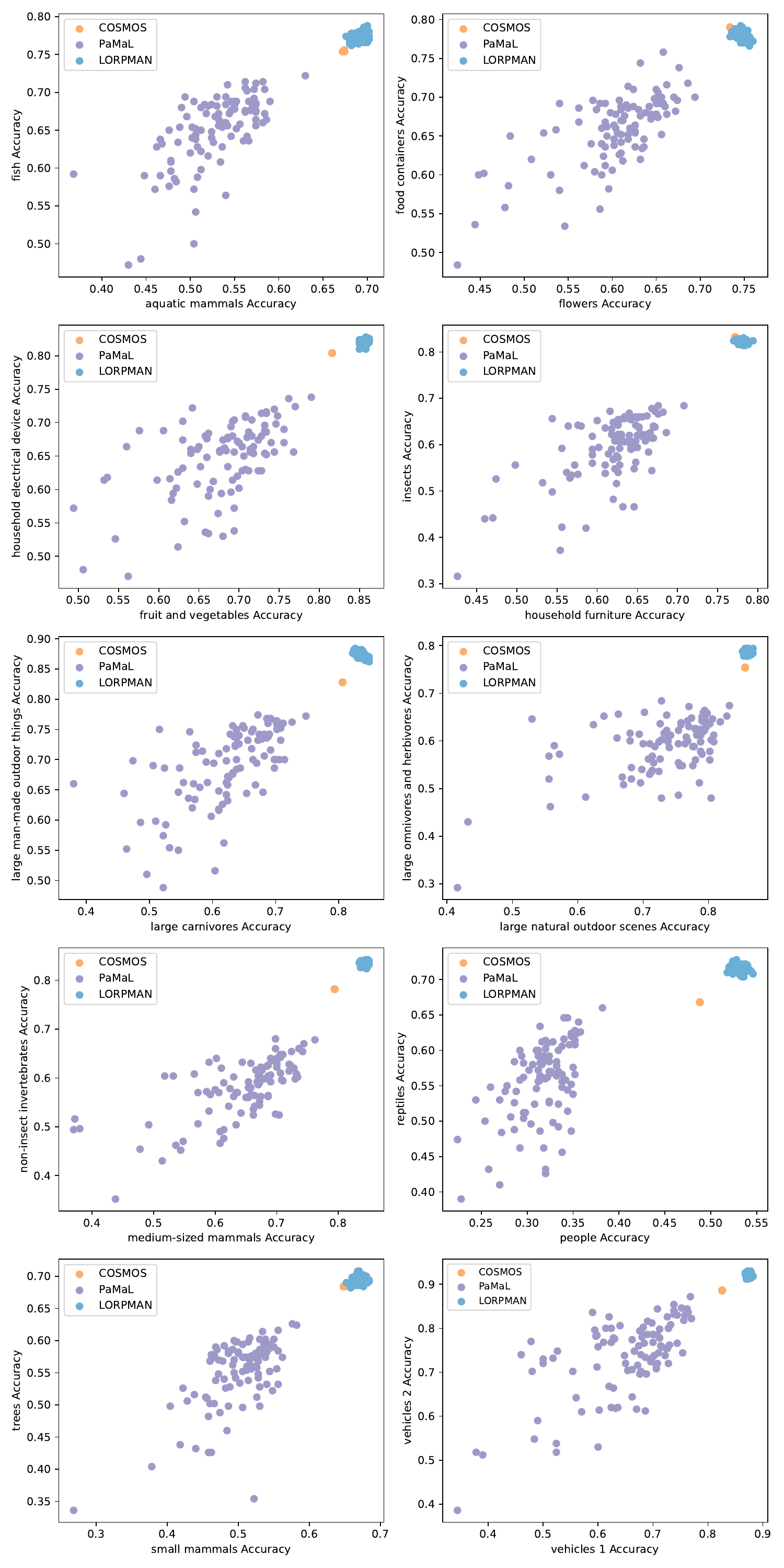}
    \caption{Test performance of COSMOS, PaMaL, and LORPMAN on {\it CIFAR-100}.}
    \label{fig:cifar100}
\end{figure}



\end{document}